\documentclass[11pt]{amsart}
\usepackage[cm]{fullpage}
\usepackage{amssymb,amsmath,amsthm,amscd,mathrsfs,graphicx}
\usepackage[cmtip,all]{xy}
\usepackage{pb-diagram}
\usepackage{tikz}

\numberwithin{equation}{section}
\newtheorem{teo}{Theorem}[section]

\newtheorem{theorem}{Theorem}[section]
\newtheorem{proposition}{Proposition}[section]
\newtheorem{lemma}{Lemma}[section]

\theoremstyle{definition}

\theoremstyle{remark}
\newtheorem{remark}[teo]{Remark}
\newcounter{q}
\counterwithin{q}{section}

\counterwithin{e}{section}

 \newtheorem{example}[q]{Example}
 
\begin{document}
\bibliographystyle{amsplain}
\author[B. Weinkove]{Ben Weinkove}
\address{Department of Mathematics, Northwestern University, 2033 Sheridan Road, Evanston, IL 60208, USA. \\ email: weinkove@math.northwestern.edu}
\thanks{Research supported in part by NSF grant DMS-2348846 and the Simons Foundation.}

\title{Stochastic neighborhood embedding \\ and the gradient flow of relative entropy}

\maketitle

\vspace{-20pt}

\begin{abstract}
Dimension reduction, widely used in science,  maps high-dimensional data into low-dimen\-sion\-al space.  We investigate a basic mathematical model underlying the techniques of stochastic neighborhood embedding (SNE) and its popular variant t-SNE.  Distances between points in high dimensions are used to define a probability distribution on pairs of points, measuring how similar the points are.  The aim is to map these points to low dimensions in an optimal way so that similar points are closer together.  This is carried out by minimizing the relative entropy between two probability distributions.  

We consider the gradient flow of the relative entropy and analyze its long-time behavior.  This is a self-contained mathematical problem about the behavior of a system of nonlinear ordinary differential equations.   We find optimal bounds for the diameter of the evolving sets as time tends to infinity.  In particular,  the diameter may blow up for the t-SNE version, but remains bounded for SNE.

\end{abstract}

\section{Introduction}

Dimension reduction (or dimensionality reduction) refers to a method of representing high-dimensional data in low-dimensional space.  The goal is to retain the essential features of the dataset, such as clustering, in the low-dimensional space (see \cite{Gbook}, for example). This paper investigates a basic mathematical model underlying two particular dimension reduction techniques:  \emph{stochastic neighborhood embedding} (SNE) and the more widely-used variant called  \emph{t-distributed stochastic neighborhood embedding} (t-SNE).  They were introduced by  Hinton and Roweis in 2002 \cite{HR} and van der Maaten and Hinton in 2008 \cite{vH}, respectively.  These techniques have been used extensively in science, including in medical research, see for example \cite{HPH, KB, LCYH}.  Given these applications, it is important to understand rigorously their mathematical foundations.  Conversely, we believe these methods may reveal some mathematical phenomena of interest in their own right.
 We  describe briefly these methods, pose some questions, and then discuss our results.
 
 \subsection{The SNE and t-SNE algorithms}
Let $x_1, \ldots, x_n$ be  $n$ points in $\mathbb{R}^d$.  Dimension reduction should provide  a map $x_i \mapsto y_i$ to  $y_1, \ldots, y_n$ in $\mathbb{R}^s$.  Since we wish to reduce the dimension of the space, we assume $s$ is strictly smaller than $d$.  We also assume that $n > s+1$, otherwise one could find an $s$-plane in $\mathbb{R}^d$ containing all the points.  (In practical applications, usually $n$ and $d$ are very large and $s=2$ or $3$).

Write $\mathcal{P}_n$ for the discrete set of $n \choose 2$ pairs representing the directed edges between $n$ points,
$$\mathcal{P}_n =\{ (i,j) \ | \   i,j=1, \ldots, n, \ i \neq j \}.$$
The informal outline of the method has three steps, each of which requires some choices.
\begin{enumerate}
\item Define a discrete probability distribution $(p_{ij})_{i\neq j}$ on $\mathcal{P}_n$ representing the ``similarity''  of the points $x_1, \ldots, x_n$ in $\mathbb{R}^d$.  We insist on the symmetry $p_{ij}=p_{ji}$, for simplicity.  The larger the probability $p_{ij}$, the ``more similar'' are the points $x_i$ and $x_j$.
\item Given any points $y_1, \ldots, y_n$ in $\mathbb{R}^s$, we define a discrete probability distribution $(q_{ij})_{i\neq j}$ on $\mathcal{P}_n$ with $q_{ij}=q_{ji}$.  The larger the probability $q_{ij}$, the closer the points $y_i$ and $y_j$. 
\item We look for points $y_1, \ldots, y_n$ so that the probability distributions $(p_{ij})$ and $(q_{ij})$ are as close together as possible, by minimizing a cost function.
\end{enumerate}

We describe now the steps more precisely.  For (1), define, for positive constants $\sigma_1, \ldots, \sigma_n$, 
$$p_{j|i} = \frac{\exp(-|x_i-x_j|^2/(2\sigma_i^2))}{\sum_{k \, | \, k \neq i}  \exp(-|x_i-x_k|^2/(2\sigma_i^2))}\quad \textrm{for } i \neq j.$$
This represents ``the conditional probability that $x_i$ would pick $x_j$ as its neighbor if neighbors were picked in proportion to their probability density under a Gaussian centered at $x_i$''  \cite{vH}, where $\sigma_i$ is the variance of the Gaussian.  The numbers $\sigma_i$ are selected according to certain criteria; since this is not the focus of the paper refer the reader to \cite{vH}.  A word about summation notation is called for.  If $i$ is fixed, we use $\sum_{k \, | \, k \neq i}$ to denote the sum over all $k=1, \ldots, n$ with $k$ not equal to $i$.  We will use $\sum_{k \neq \ell}$  to denote a double summation over all $k$ and all $\ell$ from $1$ to $n$ with $k\neq \ell$.

We define 
$$p_{ij} = \frac{p_{i|j} + p_{j|i}}{2n},$$
and one can check that this is a probability distribution on $\mathcal{P}_n$.


For step (2), for a smooth function $\beta :[0,\infty) \rightarrow (0,\infty)$, define
\begin{equation} \label{qijdefn}
q_{ij} := \frac{\beta(|y_i-y_j|^2)}{\sum_{k \neq \ell}  \beta(|y_k-y_\ell|^2)}, \quad \textrm{for } i\neq j.
\end{equation}
In the case of SNE, we take $\beta(x) = e^{-x}$, which corresponds to a Gaussian distribution.  For t-SNE, we take
$$\beta(x) = \frac{1}{1+x},$$
which corresponds to  a Student's t-distribution with one degree of freedom (a Cauchy distribution), from which t-SNE gets its name.  

For step (3), we take our cost function to be the relative entropy (or Kullback-Leibler divergence),
\begin{equation} \label{C}
\mathcal{C}(Y) := \sum_{\substack{i \neq j}} p_{ij} \log \frac{p_{ij}}{q_{ij}},
\end{equation}
for points $Y= ( y_1, \ldots, y_n )$ with $y_i \in \mathbb{R}^s$.  The relative entropy is a measure of how the $p_{ij}$ differ from the $q_{ij}$; it is 
nonnegative and vanishes if and only if the probability distributions agree.
 The goal in step (3) is to minimize $\mathcal{C}(Y)$ over all $Y$.  The standard algorithms start with random initial data $Y_0$ in $\mathbb{R}^s$ and carry out gradient descent  for a finite number of steps, adding a ``random jitter'' at each step.  There are also  variants such as ``early exaggeration''.  We refer the reader to \cite{HR, vH} for more details.

The  development of dimension reduction methods such as the SNE and t-SNE has been driven largely by empirical approaches (see \cite{WHRS} for example), with comparatively little focus on rigorous mathematical foundations.  We describe now a few of the recent theoretical papers on the subject, all of which treat  t-SNE.
Under specific conditions, it has been shown that t-SNE separates clusters in a suitable sense by Shaham and Steinerberger \cite{SS}, Arora, Hu and Kothari \cite{AHK} and Linderman and Steinerberger \cite{LS}. Analysis of the early exaggeration phase, mentioned above, was studied by Linderman and Steinerberger \cite{LS} and further developed by Cai and Ma \cite{CM}.  Steinerberger and Zhang \cite{SZ} analyzed t-SNE on a single homogeneous cluster and in this setting recast the cost function as a classical calculus of variations problem.  Auffinger and Fletcher \cite{AF} considered the case when the initial data is given by $n$ independent identically distributed outputs and showed under certain conditions the convergence of the minima of the entropy to an equilibrium distribution.  Recently, Jeong and Wu \cite{JW} gave some conditions under which the solution of the gradient flow of the relative entropy remains bounded.

\subsection{The gradient flow of relative entropy}
The focus of this paper is the following mathematical problem, which can be described in a simple, self-contained way.  Let $( p_{ij} )_{i\neq j}$ be a given probability distribution on $\mathcal{P}_n$, with $n>s+1$ and $p_{ij}>0$ for all $i\neq j$.  Fix a smooth decreasing function $\beta: [0,\infty) \rightarrow (0,\infty)$ such that $\sup_{x\in[0,\infty)}|(\log \beta)'(x)|<\infty$.  Define probabilities $q_{ij}$ by (\ref{qijdefn}) and let $\mathcal{C}$ be the relative entropy (\ref{C}).  We wish to understand the minima of $\mathcal{C}$.  To do so, we consider the gradient flow of $\mathcal{C}$ with arbitrary initial data $Y_0 \subset \mathbb{R}^s$, which is given by
 \begin{equation} \label{gradflow}
 \begin{split}
\frac{dy_i}{dt}  = {} & 4\sum_{\substack{j \, | \,  j \neq i}}  (p_{ij} - q_{ij})(y_i-y_j) (\log \beta)'(|y_i - y_j|^2),
\end{split}
\end{equation}
for $\beta(x)$ as described in step 2.  See Section \ref{sectiongrad} for a proof of this formula.
There exists a solution $Y(t)=(y_1(t), \ldots, y_n(t))$ for all time $t \ge 0$ by Gronwall's inequality, since $(\log \beta)'$ is bounded by assumption.  We can now state our main problem.

\medskip

\noindent
{\bf Problem.} \ 
What is the behavior of solutions $Y(t)$ of this flow as $t \rightarrow \infty$ and how does this relate to the minima of $\mathcal{C}$?  Does $Y(t)$ converge (after rescaling, if necessary) to a set of points  $Y_{\infty} \subset \mathbb{R}^s$?  How does the limit depend on the initial data $Y_0$?

\medskip

\begin{remark} \, 
\begin{enumerate}
\item[(i)] These questions remain completely open, as far as we know.  In this paper, our main results are to obtain optimal estimates on the diameter of $Y(t)$ as $t \rightarrow \infty$ when $\beta(x)= (1+x)^{-1}$ (t-SNE) and  $\beta(x)=e^{-x}$ (SNE).  We are also interested in more general functions $\beta$.
\item[(ii)] As in \cite{JW}, we consider the gradient flow (\ref{gradflow}) rather than a gradient descent method with a finite number of steps, since we regard it as more natural from a mathematical point of view.  The behavior of $Y(t)$ as $t \rightarrow \infty$ should be closely related to the behavior of the SNE and t-SNE algorithms.
\item[(iii)] We regard the $p_{ij}$ as fixed at the outset, ignoring the construction  in step (1) above, although this is surely a very important aspect to be studied.  On the other hand, this means that our results are relevant no matter how the $p_{ij}$ are chosen.  In fact, our problem makes no reference to the high-dimensional space $\mathbb{R}^d$; rather it is a question about prescribing as close as possible a probability distribution on $\mathcal{P}_n$ arising from a configuration of $n$ points $y_1, \ldots, y_n$ in $\mathbb{R}^s$.  
\end{enumerate}
\end{remark}

\subsection{Results} We  explore the question of what happens to the diameter of the solution sets $Y(t)$  of the gradient flow of relative entropy, as $t \rightarrow \infty$.  We find optimal bounds when $\beta(x) = (1+x)^{-1}$ and $\beta(x)=e^{-x}$, the t-SNE and SNE cases, respectively. We observe very different behavior in these two cases.  

Write $Y(t)$ for the solution of the  flow starting at $Y_0=Y(0)$.   Our first theorem is as follows.

\pagebreak[3]
\begin{theorem} \label{mainthm0} \, The following diameter bounds hold.
\begin{enumerate}
\item[(i)] If $\beta(x) = (1+x)^{-1}$,   then for $t \ge 1$,
\begin{equation} \label{moptimal}
\emph{diam}\, Y(t) \le C  t^{\frac{1}{4}}.
\end{equation}
\item[(ii)]
If $\beta(x) = e^{-x}$,   then for $t \ge 0$,
\begin{equation} \label{eoptimal}
\emph{diam}\, Y(t) \le C.
\end{equation}
\end{enumerate}
\end{theorem}

Here and in the sequel, we use $C, C', c, c'$ etc to denote ``uniform'' positive constants, which means they are independent of $t$, but may depend on $n$, $s$, $Y_0$, $p_{ij}$ and the choice of function $\beta$.  These constants may differ from line to line.

The bounds of Theorem \ref{mainthm0} are optimal  in terms of $t$ in the following sense.

\begin{theorem} \label{mainthm2} \,
\begin{enumerate}
\item[(i)] Assume $\beta(x) = (1+x)^{-1}$.  We can find a probability distribution $(p_{ij})$  on some $\mathcal{P}_n$, initial data $Y_0 \subset \mathbb{R}^s$ and a positive constant $c$ such that $\emph{diam}\,  Y(t) \ge  c\, t^{\frac{1}{4}}$ for $t\ge 0$.
\item[(ii)] Assume $\beta(x) = e^{-x}$.  We can find a probability distribution $(p_{ij})$ on some $\mathcal{P}_n$, initial data $Y_0 \subset \mathbb{R}^s$ and a positive constant $c$ such that $\emph{diam}\,  Y(t) \ge  c$ for $t\ge 0$.
\end{enumerate}
\end{theorem}

The examples in the proof of this result are very simple and have $n=3$ and $s=1$.  We also give an example (Example \ref{example2dim} below) with $n=4$ and $s=2$, to show that this is not a phenomenon unique to $s=1$.

We now consider the case of more general $\beta$. 
We make the following assumptions:
\begin{enumerate}
\item[(A1)] $\gamma(x)$ defined by $\gamma(x) :=  \frac{1}{\beta(x)}$ is a smooth convex function satisfying
$$\gamma(0)=1, \quad \gamma'(x) \ge 0, \quad \lim_{x\rightarrow \infty} \gamma(x)=\infty.$$
\item[(A2)] $\sup_{x\in [0,\infty)} (\log \gamma)'(x)<\infty.$ 
\end{enumerate}
Given (A1),  assumption (A2) is equivalent to the boundedness of $(\log \beta)'$.  The functions $\beta(x) = (1+x)^{-1}$ and $\beta(x) = e^{-x}$ satisfy (A1) and (A2).

Under these assumptions, we prove the following.

\begin{theorem} \label{thmconv} 
If $\emph{diam}\, Y(t_i) \rightarrow \infty$ for  $t_i \rightarrow \infty$ then the sequence
$$\frac{Y(t_i)}{\emph{diam} \,Y(t_i)}$$
subconverges to $n$ distinct points $Y_{\infty} = ( y_1^{\infty}, \ldots, y_n^{\infty} )$ in $\mathbb{R}^s$.
\end{theorem}

In the special case $\beta(x)=(1+x)^{-1}$, Theorem \ref{thmconv} is a consequence of  a result of Jeong and Wu \cite{JW}  (see Proposition \ref{prop} below).

It is easy to find examples where $\textrm{diam}\, Y(t)$ remains bounded in $t$, even when $\beta(x) = (1+x)^{-1}$.  The following gives rise to a large class of examples.

\begin{theorem} \label{thmdouble}
If $p_{1j} = p_{2j}$ for all $j>2$ and $y_1=y_2$ in $Y_0$ then there exists a uniform constant $C$ such that
$$\emph{diam}\, Y(t) \le C.$$
\end{theorem}

In particular, if $\textrm{diam} \, Y(t)$ is unbounded, one can add an extra point to this data (``doubling up'' one of the existing points)  and the diameter will remain bounded. 

We also give an example (see Example \ref{examplecollapse} below) where $\textrm{diam}\, Y(t) \rightarrow 0$ as $t \rightarrow \infty$.

\subsection{SNE versus t-SNE} Our results show a marked difference in the behavior of the flow $Y(t)$ in the SNE versus t-SNE cases.  Only  t-SNE exhibits divergence of points at infinity, and Theorem \ref{thmconv} shows that if this occurs one obtains distinct points in the rescaled limit.  We believe this is related to the so-called ``crowding problem'' discussed by van der Maaten and Hinton  in \cite[Section 3.2]{vH}.    They note that in SNE, the points $y_i$ tend to be pushed towards the center, preventing gaps from occurring between clusters.  They give heuristic reasons why the heavier tails of the probability distribution in t-SNE may compensate for this effect, pushing points apart.  Indeed, this motivated their construction of the t-SNE algorithm.

\medskip

The outline of the paper is as follows.  In Section \ref{sectiongrad} we compute the formula (\ref{gradflow}) for the gradient flow.  Section \ref{sectionproofs} is the heart of the paper, in which we prove the main results, Theorems \ref{mainthm0} to \ref{thmdouble}.  In Section \ref{sectionexamples}, we give two further examples.  Finally in Section \ref{questions}, we pose some further questions for future study.

\medskip \noindent
{\bf Acknowledgements.} \ 
The author is very grateful to Antonio Auffinger for spurring his interest in the mathematical study of t-SNE, and for some helpful comments.  The author also thanks the referees for suggesting some improvements to the exposition.

\section{The gradient flow of relative entropy} \label{sectiongrad}

In this section we prove the formula (\ref{gradflow}) for the gradient flow of the entropy $\mathcal{C}$. The  gradient of $\mathcal{C}$ was calculated for $\beta=e^{-x}$  and $\beta(x) = (1+x)^{-1}$ in \cite{HR, vH} (see also \cite[Chapter 16]{Gbook}) and its extension to general $\beta$ is straightforward.  However, we include the details here for the convenience of the reader\footnote{The author thanks one of the referees for pointing out a simplification of his original argument.}.  Recalling (\ref{qijdefn}), we write
\begin{equation} \label{firstZ}
q_{ij} = \frac{\beta(|y_i-y_j|^2)}{Z}, \quad \textrm{for }Z:= \sum_{i\neq j} \beta(|y_j - y_j|^2).
\end{equation}
We have
\[
\begin{split}
\nabla_{y_m} \mathcal{C} = {} & - \sum_{i \neq j} p_{ij} \nabla_{y_m} \log q_{ij} \\
= {} & - \sum_{i \neq j} p_{ij} \left( \nabla_{y_m} \log \beta (|y_i-y_j|^2) - \nabla_{y_m} \log Z \right) \\
= {} & - \sum_{i \neq j} p_{ij} (\log \beta)'(|y_i-y_j|^2) \nabla_{y_m} |y_i-y_j|^2 + \frac{1}{Z} \nabla_{y_m} \sum_{i \neq j} \beta(|y_i-y_j|^2),
\end{split}
\]
where for the last line we used the fact that $\sum_{i\neq j} p_{ij}=1$.  
Observe that
$$\nabla_{y_m} |y_i-y_j|^2 = 2\delta_{mi} (y_m-y_j) + 2\delta_{mj} (y_m-y_i).$$
Hence
\[
\begin{split}
\nabla_{y_m} \mathcal{C} = {} & - 4 \sum_{j \, | \, j\neq m} p_{mj} (y_m-y_j) (\log \beta)'(|y_m-y_j|^2) + \frac{4}{Z} \sum_{j \, | \, j\neq m} (y_m-y_j) \beta'(|y_m-y_j|^2)  \\
= {} & - 4 \sum_{j \, | \, j \neq m} (p_{mj} - \frac{\beta(|y_m-y_j|^2)}{Z} ) (y_m-y_j)  (\log \beta)'(|y_m-y_j|^2) \\
= {} & -4 \sum_{j \, | \, j \neq m} (p_{mj} - q_{mj}) (y_m-y_j) (\log \beta)'(|y_m-y_j|^2),
\end{split}
\]
and then (\ref{gradflow}) follows.

\section{Proofs of the main results} \label{sectionproofs}

We begin with the most general setting, and then specialize later to the cases of $\beta=(1+x)^{-1}$ and $\beta(x)=e^{-x}$.
Assume that the function
$\gamma(x):= \frac{1}{\beta(x)}$ satisfies conditions (A1) and (A2) as in the introduction.

As observed in \cite{JW}, the center of mass of the points $y_1, \ldots, y_n$ does not change in $t$.  Indeed,
\[
\begin{split}
\frac{d}{dt} \sum_{i=1}^n y_i= {} &   4 \sum_{j \neq i}  (p_{ij} - q_{ij}) (y_i-y_j) (\log \beta)'(|y_i-y_j|^2) =0,
\end{split}
\]
since $(y_i-y_j)$ is anti-symmetric in $i,j$ while $(p_{ij}-q_{ij})(\log \beta)'(|y_i-y_j|^2)$ is symmetric in $i$, $j$.
 We may and do assume from now on that the center of mass of $Y$ is the origin.

Define $S: = \sum_{i=1}^n |y_i|^2$.  We note that since the center of mass of $Y$ is the origin,
\begin{equation} \label{diamS}
C^{-1} \textrm{diam}\, Y \le \sqrt{S} \le C \textrm{diam}\, Y.
\end{equation}

We will later make use of the following, which was already proved by Jeong and Wu \cite{JW} in the case $\beta=(1+x)^{-1}$.
\begin{proposition} \label{prop}
There exist uniform constants $C$ and $c>0$ such that if $S \ge C$ then
\begin{equation} \label{claimeqn}
|y_i-y_j|^2 \ge c S, \quad \textrm{for } i \neq j.
\end{equation}
\end{proposition}

We first prove an elementary lemma, which uses assumption (A1).

\begin{lemma} \label{lemma}
Given $C\ge 1$, there exists $C_0$ depending only on $\beta$ and $C$ such that if  $z \ge C_0$ and
$\beta(w) \le C \beta (z),$
then $$w \ge \frac{z}{2C}.$$
\end{lemma}
\begin{proof}
Suppose $z \ge C_0$.  The convexity of $\gamma$ implies that for any $\tau \in [0,1]$,
$$\gamma((1-\tau)z) \le \tau \gamma(0) + (1-\tau)\gamma(z).$$
Choose $\tau=1-\frac{1}{2C}$ and choose $C_0$ large enough so that $\gamma(z) \ge 2C$.  Then we have
$$\gamma\left( \frac{z}{2C} \right) < 1+ \frac{1}{2C} \gamma(z) \le \frac{1}{2C} \gamma(z)+ \frac{1}{2C} \gamma(z) = \frac{1}{C} \gamma(z).$$
Hence,
$$\beta\left( \frac{z}{2C} \right) > C \beta(z) \ge \beta(w),$$
where the second inequality follows from the assumption.  Since $\beta$ is decreasing, we have $w \ge \frac{z}{2C}$. 
\end{proof}

We now prove the proposition.

\begin{proof}[Proof of Proposition \ref{prop}] It must be true that for \emph{some} $i\neq j$ we have $|y_i-y_j|^2 \ge cS$ since the center of mass of the points is the origin.  Without loss of generality,  assume that $|y_1-y_2|^2 \ge cS$.  To prove (\ref{claimeqn}) holds for all $i\neq j$ we argue as follows.  Since the gradient flow decreases the functional
$$\mathcal{C}(Y) = \sum_{i\neq j} p_{ij} \log \frac{p_{ij}}{q_{ij}},$$
we have $\mathcal{C}(Y(t)) \le \mathcal{C}(Y_0)$ and hence,
$$-\sum_{i\neq j} p_{ij} \log q_{ij} \le -\sum_{i \neq j} p_{ij} \log p_{ij} + \mathcal{C}(Y_0) \le C,$$
for a uniform constant $C$.   Since each term $-p_{ij} \log q_{ij}$ in the sum is positive, we have for $i\neq j$, 
$$-p_{ij} \log q_{ij} \le C.$$
This implies in particular that $$\log q_{12} \ge -C/p_{12} =: -C'$$ for a uniform $C'$.  Here we recall our assumption that each $p_{ij}$ (for $i\neq j$) is strictly positive and our uniform constants may depend on the $p_{ij}$.  Exponentiating gives
\begin{equation}
c' \le q_{12} = \frac{\beta(|y_1-y_2|^2)}{\sum_{i\neq j} \beta (|y_i-y_j|^2)},
\end{equation}
for a uniform  positive $c'$.   

For each $i\neq j$,  we have
$$\beta (|y_i-y_j|^2) \le C \beta (|y_1-y_2|^2).$$
Applying the lemma we see that if $S$ is sufficiently large then
$$|y_i-y_j|^2 \ge \frac{|y_1-y_2|^2}{2C} \ge \frac{cS}{2C},$$
as required.
\end{proof}

We can now prove Theorem \ref{mainthm0}.

\begin{proof}[Proof of Theorem \ref{mainthm0}] 
Compute
\[
\begin{split}
 \frac{d}{dt} \sum_{i=1}^n |y_i|^2 = {} & 8 \sum_{i \neq j} (p_{ij} - q_{ij}) y_i \cdot (y_i-y_j) (\log \beta)' (|y_i-y_j|^2) \\
 = {} & 8 \sum_{i < j} (p_{ij} - q_{ij}) y_i \cdot (y_i-y_j) (\log \beta)' (|y_i-y_j|^2) \\ {} & +8 \sum_{i  < j} (p_{ij} - q_{ij}) y_j \cdot (y_j-y_i) (\log \beta)' (|y_i-y_j|^2) \\
 = {} & 8 \sum_{i < j} (p_{ij} - q_{ij}) |y_i-y_j|^2 (\log \beta)'(|y_i-y_j|^2),
 \end{split}
 \]
 where $\sum_{i<j}$ denotes the double summation of $i$ and $j$ with $i<j$.
Now for $i\neq j$, write
$$A_{ij} = \beta(|y_i-y_j|^2),$$
and as in  (\ref{firstZ}) above,
$$q_{ij} = \frac{A_{ij}}{Z}, \quad Z:= 2\sum_{k< \ell} A_{k\ell}.$$
Then,
\begin{equation} \label{ddtS}
\begin{split}
\frac{d}{dt} S = {} &  \frac{8}{Z} \sum_{i<j} \left( p_{ij}Z  - A_{ij} \right)  |y_i-y_j|^2 (\log \beta)'(|y_i-y_j|^2).
 \end{split}
\end{equation}
In case (i),  we compute
$$x(\log \beta)'(x) = \beta(x)-1,$$
and hence
\[
\begin{split}
\frac{d}{dt} S = {} &  \frac{8}{Z} \sum_{i<j} \left( p_{ij}Z - A_{ij} \right) (A_{ij}-1) \\
= {} & \frac{8}{Z} \sum_{i<j} \left( p_{ij}Z - A_{ij} \right) A_{ij},
\end{split}
\]
since
\begin{equation} \label{equals0}
\sum_{i<j} (p_{ij}Z - A_{ij} )= \frac{1}{2}Z - \sum_{i<j} A_{ij}=0.
\end{equation}

Hence for $S$ large,
\[
\begin{split}
\frac{d}{dt} S \le {} &  \frac{C\sum_{i\neq j} A_{ij}^2}{\sum_{i \neq j} A_{ij}} 
\le   C' \sum_{i\neq j} A_{ij}  = C' \sum_{i\neq j} \frac{1}{1+ |y_i-y_j|^2} 
\le \frac{C}{S},
\end{split}
\]
where we used Proposition \ref{prop} for the last inequality.  Hence
$$\frac{d}{dt} S^2 \le C,$$
giving $S^{2} \le Ct$ and the bound (\ref{moptimal}) follows from (\ref{diamS}).

In case (ii), $$x(\log \beta)'(x) = -x = \log \beta(x),$$
and so from (\ref{ddtS}),
\begin{equation} \label{ddtSe}
\begin{split} 
\frac{d}{dt} S = {} & \frac{8}{Z} \sum_{i<j} \left( p_{ij} Z - A_{ij} \right) \log A_{ij} \\
= {} &  \frac{8}{Z} \sum_{i<j} \left( p_{ij} Z - A_{ij} \right) \log \frac{A_{ij}}{Z},
\end{split}
\end{equation}
using again (\ref{equals0}).

Next we claim that there is a universal constant $\eta=\eta(s)>0$ such that for any $n$ points $y_1, \ldots, y_n \in \mathbb{R}^s$ with $n>s+1$,
\begin{equation} \label{c}
\max |y_i-y_j|^2 \ge (1+\eta) \min_{k\neq \ell} |y_k - y_{\ell}|^2.
\end{equation}
Indeed, without loss of generality we may assume that  $\min_{k\neq \ell} |y_k - y_{\ell}|^2=1$ (if $\min_{k\neq \ell} |y_k - y_{\ell}|^2=0$ there is nothing to prove).  If the claim is false then we can find a sequence  $( y^{(\ell)}_1, \ldots, y^{(\ell)}_{n} )_{\ell=1}^{\infty}$  such that $1\le |y^{(\ell)}_i-y^{(\ell)}_j|^2 \le (1+\ell^{-1})$ for all $i,j$.  Letting $\ell \rightarrow \infty$ produces $n>s+1$ points in $\mathbb{R}^s$ which are all equidistant from each other.  This contradicts an elementary result that the equilateral dimension of Euclidean space $\mathbb{R}^s$ is $s+1$ (see \cite{G}, for example).

We will now show that there are uniform positive constants $c$ and $C$ such that for $S$ sufficiently large
\begin{equation} \label{goal}
\frac{d}{dt} S \le -cS +C.
\end{equation}
Given this we are done since it implies that $S$ decreases when it is too large, and hence $S$ must be bounded.

Working with $Y(t)=(y_1(t), \ldots, y_n(t))$ at a fixed $t$, we may assume without loss of generality by the claim (\ref{c}) above that 
$$|y_1-y_2|^2 \ge (1+\eta) \min_{k\neq \ell} |y_k - y_{\ell}|^2.$$
Then, since $S$ is large, using Proposition \ref{prop},
$$A_{12} = e^{-|y_1-y_2|^2} \le e^{-(1+\eta)\min |y_k-y_{\ell}|^2}\le e^{-c'S}e^{-\min|y_k-y_{\ell}|^2} < e^{-c'S} Z.$$
In particular, $A_{12}$ is very small compared to $Z$.
From (\ref{ddtSe}) we have
\[ 
\begin{split} 
\frac{d}{dt} S = {} & \frac{8}{Z} \left( p_{12} Z - A_{12} \right) \log \frac{A_{12}}{Z} + \frac{8}{Z} \sum_{\substack{i<j  \, | \, (i,j) \neq (1,2)}} \left( p_{ij} Z - A_{ij} \right) \log \frac{A_{ij}}{Z} \\
\le {} & -\frac{4c'}{Z} p_{12} Z S +  \frac{8}{Z} \sum_{\substack{ i<j \, | \, (i,j)\neq (1,2), \\ p_{ij}Z-A_{ij} <0}} \left( p_{ij} Z - A_{ij} \right) \log \frac{A_{ij}}{Z},
\end{split}
\]
since we may assume that $p_{12}Z - A_{12} \ge \frac{p_{12}}{2} Z$, and noting that $\log \frac{A_{ij}}{Z} \le 0$ so that we can discard terms in the sum with $p_{ij} Z - A_{ij} \ge 0$.  Then for $S$ large,
\[
\begin{split} 
\frac{d}{dt} S = {} & -cS + \frac{8}{Z} \sum_{\substack{ i<j \,  | \,  (i,j)\neq (1,2), \\ p_{ij}Z-A_{ij} <0}} \left( p_{ij} Z - A_{ij} \right) \log p_{ij} \\
\le {} & -c S+ C <0,
\end{split}
\]
proving (\ref{goal}) as required.
\end{proof}

\begin{remark}
It was pointed out to the author by Antonio Auffinger that if one is interested in the Langevin dynamics instead of the gradient flow, the same computation can be used to 
estimate the expectation of $S$.
\end{remark}

\begin{proof}[Proof of Theorem \ref{mainthm2}]
We provide an example with three points in $\mathbb{R}$.  Write $p_{12}=p_{23}=a$ so that $p_{13}=1/2(1-4a)$ with $a\in (0,1/4)$ to be determined.  Assume that $y_1(t)=X(t)$, $y_2(t)=0$ and $y_3(t) = -X(t)$ so that the points are symmetric about the origin, with $X(0)>0$.  Recalling (\ref{firstZ}), we have $Z = 4\beta(X^2)+2\beta(4X^2)$ and
\begin{equation} \label{ddtXs}
\begin{split}
\frac{d}{dt} X = {} & 4\left( a- \frac{\beta(X^2)}{4\beta(X^2)+2\beta(4X^2)}  \right) X (\log \beta)'(X^2)  \\ {} & +4\left( \frac{1}{2}(1-4a) - \frac{\beta(4X^2)}{4\beta(X^2) + 2\beta(4X^2)}   \right) 2X(\log \beta)'(4X^2) \\
= {} & \frac{4X}{Z} \left( (1-4a) \beta(X^2) - 2a\beta(4X^2)\right)  \frac{\gamma'(X^2)}{\gamma(X^2)} \\ 
{} & + \frac{8X}{Z} \left( 4a \beta(4X^2) -2(1-4a) \beta(X^2)  \right)  \frac{\gamma'(4X^2)}{\gamma(4X^2)} \\
= {} & \frac{4X}{Z} \frac{1}{(\gamma(X^2))^2(\gamma(4X^2))^2} \bigg[ \left( (1-4a) \gamma(4X^2) - 2a\gamma(X^2) \right) \\ {} & \cdot \left( \gamma'(X^2) \gamma(4X^2) - 4\gamma'(4X^2) \gamma(X^2) \right) \bigg].
\end{split}
\end{equation}
For (i) we have $\gamma(x)=1+x$ and the term in the square brackets is
\begin{equation} \label{brackets}
\begin{split}
[ \cdots ] = {} &  -3 \left(1-6a + (4(1-4a) - 2a) X^{2} \right).
\end{split}
\end{equation}
Now choose $a \in(2/9,1/4)$ so that $1-6a<0$ and $4(1-4a) - 2a<-c$ for a uniform $c>0$.  Then
$$[\cdots ] \ge cX^{2}.$$
It follows that for $X$ large,
$$\frac{d}{dt} X \ge \frac{cX X^{2}}{X^{-2} X^{4} X^{4}} = cX^{-3},$$
and hence $X^{4} \ge ct$.  Then for $t \ge 1$,
$$\textrm{diam}\, Y(t) = 2X(t) \ge ct^{\frac{1}{4}},$$
giving the example for (i).

For (ii), choose $a \in (1/6, 1/4)$.  Then the term in the square brackets in (\ref{ddtXs}) is
$$[\cdots ] = - 3e^{5X^2} ((1-4a)e^{4X^2} - 2ae^{X^2}),$$
which is negative for $X$ large and positive for $X$ small.  This implies that the $\textrm{diam}\, Y(t)$ is bounded from below away from zero, as required.
\end{proof}

For Theorem \ref{thmconv} we argue as follows.

\begin{proof}[Proof of Theorem \ref{thmconv}] The only nontrivial assertion is that points in $Y_{\infty}$ are distinct, but this follows immediately from Proposition \ref{prop}.\end{proof}

Finally, we prove Theorem \ref{thmdouble}.

\begin{proof}[Proof of Theorem \ref{thmdouble}]
We first show that $y_1(t)=y_2(t)$ for all $t$.  Compute
\[
\begin{split}
\frac{d}{dt} |y_1-y_2|^2  = {} & 8 (y_1-y_2)\cdot \bigg[ (p_{12}-q_{12})(y_1-y_2) (\log \beta)' (|y_1-y_2|^2) \\
{} & - (p_{12} - q_{12}) (y_2-y_1) (\log \beta)' (|y_1-y_2|^2) \\
{} & +  \sum_{j >2} (p_{1j} - q_{1j}) (y_1-y_j) (\log \beta)' (|y_1-y_j|^2) \\
{} & -  \sum_{j >2} (p_{2j} - q_{2j}) (y_2-y_j) (\log \beta)' (|y_2-y_j|^2) \bigg].
\end{split}
\]
Consider $Y(t)$ for $t \in [0,T]$ where $T>0$ is fixed.  In what follows, $C, C'$ will denote constants that may now depend also on $T$ and $Y(t)$ for $t \in [0,T]$.
By the Mean Value Theorem,
$$|(\log \beta)'(|y_1-y_j|^2) - (\log \beta)'(|y_2-y_j|^2)| \le C  \big| |y_1-y_j|^2 -  |y_2-y_j|^2 \big| \le C' |y_1-y_2|,$$
and similarly
$$|q_{1j} - q_{2j}| \le C |y_1-y_2|.$$
Hence, using the fact that $p_{1j}=p_{2j}$ for all $j>2$,
$$\frac{d}{dt} |y_1-y_2|^2 \le C |y_1-y_2|^2,$$
and thus $e^{-Ct}|y_1-y_2|^2$ is decreasing on $[0,T]$ and initially vanishes.  This implies that $|y_1-y_2|^2$ remains zero for $t \in [0,T]$.  Hence $y_1(t)=y_2(t)$ on $[0,T]$ and since $T$ was arbitrary this proves $y_1(t)=y_2(t)$ for all $t$.

The theorem then follows immediately from Proposition \ref{prop}.
\end{proof}

\begin{remark}
Note that in the example given in the proof of Theorem \ref{mainthm2} for $\beta(x) = (1+x)^{-1}$, if we instead chose the initial data with $y_1=y_3$ then it follows from Theorem \ref{thmdouble} that $\textrm{diam} \, Y(t)$ is bounded.  Hence for a given probability distribution $\{ p_{ij} \}$, whether $\textrm{diam} \, Y(t)$ tends to infinity or not may depend on the initial data $Y_0$.
\end{remark}

\section{Examples} \label{sectionexamples}

\begin{example} \label{examplecollapse}  This is an example of three points in $\mathbb{R}$ which collapse to the origin as $t$ tends to infinity.

 Take $\beta(x)=(1+x)^{-1}$ and take $s=1$ and $p_{12}=p_{13}=p_{23}=1/6$.  Assume that $y_1(t) = X(t), \ y_2(t)=0$ and $y_3(t) = - X(t)$ so that points are symmetric about the origin, and we take $X(0)=1$.  From (\ref{ddtXs}) and (\ref{brackets}) with $a=1/6$ we have
$$\frac{dX}{dt} = - \frac{12X}{Z} \frac{1}{(\gamma(X^2))^2(\gamma(4X^2))^2}  X^{2}.$$
Then we see that $0 \le X(t) \le 1$ and hence
$$\frac{dX}{dt} \le  -CX^3.$$
It follows that 
$$\frac{d}{dt} X^{-2}  \ge \frac{1}{C},$$
which gives 
$$\textrm{diam}\, Y(t) = 2X(t) \le Ct^{-\frac{1}{2}} \rightarrow 0 \quad \textrm{as } t\rightarrow \infty.$$

If $\beta(x) = e^{-x}$ then
$$\frac{dX}{dt} = - \frac{4X}{Z} \frac{1}{(\gamma(X^2))^2(\gamma(4X^2))^2} e^{6X^2} (e^{3X^2}-1).
$$
It follows that $0 \le X(t) \le 1$ and since $e^{3X^2}-1 \ge 3X^2$, we have
$$\frac{d X}{dt} \le -C X^3$$
and hence 
$$\textrm{diam}\, Y(t) = 2X(t) \le Ct^{-\frac{1}{2}} \rightarrow 0 \quad \textrm{as } t\rightarrow \infty.$$
\end{example}

\begin{example} \label{example2dim} This is an example of four points in $\mathbb{R}^2$.  First assume $\beta(x) = (1+x)^{-1}$.  We will see that the flow has similar long time behavior as in the case of 3 points in $\mathbb{R}$ described in the proof of Theorem \ref{mainthm2}.  Consider the four points $y_1, \ldots, y_4$ in $\mathbb{R}^2$ with coordinates
$$y_1 = (X,0), \ y_2 = (0,X), \ y_3 = (-X, 0), \ y_4= (0, -X),$$
for $X=X(t)$, and $X(0)>0$. Define
\[
\begin{split}
& p_{12}=p_{23}=p_{34}=p_{41}=a \\
& p_{13}=p_{24} = \frac{1}{4} (1-8a),
\end{split}
\]
for $a \in (0,1/8)$ to be determined.
We have $Z=  8\beta(2X^2) + 4\beta (4X^2)$. Compute
$$\frac{dX}{dt} =  4\sum_{j=2}^4 \left(p_{1j} - \frac{\beta(|y_1-y_j|^2)}{Z}  \right) (y_1-y_j)_1 (\log \beta)' (|y_1-y_j|^2),$$
where $(y_1-y_j)_1$ means the $x_1$ component of $y_1-y_j$.

Compute
\[
\begin{split}
\frac{dX}{dt} = {} & 8 \left(  \frac{\beta(2X^2)}{Z} - a \right)  X\frac{\gamma'(2X^2)}{\gamma(2X^2)} + 4\left( \frac{\beta(4X^2)}{Z} - \frac{1}{4} (1-8a) \right) 2X \frac{\gamma'(4X^2)}{\gamma(4X^2)} \\
= {} & \frac{8X}{Z} \frac{1}{(\gamma(2X^2))^2 (\gamma(4X^2))^2} \bigg[ \left( (1-8a)\gamma(4X^2) - 4a\gamma(2X^2) \right) \\ {} & \cdot \left( \gamma'(2X^2) \gamma(4X^2) - 2\gamma'(4X^2)\gamma(2X^2) \right) \bigg].
\end{split}
\]
Using now $\gamma(x)=1+x$, the term in the square brackets is
\[
\begin{split}
[ \cdots ] = {} &  - \left(1-12a + (4(1-8a)-8a) X^{2} \right).
\end{split}
\]
Choose $a \in(1/10,1/8)$ so that $1-12a<0$ and $4(1-8a)-8a<-c$ for a uniform $c>0$.  Then
$$[\cdots ] \ge cX^{2}$$
and
$$\textrm{diam}\, Y(t) = 2X(t) \ge ct^{\frac{1}{4}},$$
as in the proof of Theorem \ref{mainthm2}.

If $\beta(x) = e^{-x}$ then by choosing $a \in (1/12, 1/8)$ one can check that the diameter remains bounded from below away from zero.
\end{example}

\section{Further questions} \label{questions}

We are still at the early stages of understanding the behavior of solutions $Y(t)$ of the gradient flow of relative entropy.  Some follow-up questions include:

\begin{enumerate}
\item[(i)] If $\textrm{diam}\, Y(t) \rightarrow \infty$ as $t\rightarrow \infty$, is this true also when $Y_0$ is perturbed slightly?
\item[(ii)] Is the limit $Y_{\infty}$ in Theorem \ref{thmconv} independent of the choice of subsequence of times?  
\item[(iii)] In the case $\beta(x) = (1+x)^{-1}$, are there probability distributions $(p_{ij})$ such that $\textrm{diam}\, Y(t) \rightarrow \infty$ for generic initial data $Y_0$?
\item[(iv)] In the case $\beta(x)=e^{-x}$ what can we say about the limits $Y(t_i) \rightarrow Y_{\infty}$?  Under what assumptions are the elements of $Y_{\infty}$ distinct?
\item[(v)] If $Y(t)$ converges to a critical point as $t \rightarrow \infty$, what is the rate of convergence?  (This is of practical importance for the users of the t-SNE algorithm).
\end{enumerate}

Given a probability distribution $(p_{ij})$, we expect the space of all possible limits $Y_{\infty}$, with varying initial data $Y_0$,   to be quite complicated in general.

\end{document}